\documentclass{opt2021} 
\usepackage{marginnote}

\newcommand{\R}[0] {\rm I\!R}

\title[The Geometric Occam's razor Implicit in Deep Learning]{The Geometric Occam's Razor Implicit in Deep Learning}

\optauthor{%
\Name{Benoit Dherin}\thanks{equal contribution} \Email{dherin@google.com}\\
\addr Google, Dublin
\AND
\Name{Michael Munn}\footnotemark[1] \Email{munn@google.com}\\
\addr Google, New York
\AND
\Name{David G.T. Barrett} \Email{barrettdavid@deepmind.com}\\
\addr DeepMind, London
}

\begin{document}

\maketitle

\begin{abstract}%
In over-parameterized deep neural networks there can be many possible parameter configurations that fit the training data exactly. However, the properties of these interpolating solutions are poorly understood. We argue that over-parameterized neural networks trained with stochastic gradient descent are subject to a Geometric Occam's Razor; that is, these networks are implicitly regularized by the geometric model complexity. For one-dimensional regression, the geometric model complexity is simply given by the arc length of the function. For higher-dimensional settings, the geometric model complexity depends on the Dirichlet energy of the function. We explore the relationship between this Geometric Occam's Razor, the Dirichlet energy and other known forms of implicit regularization. Finally, for ResNets trained on CIFAR-10, we observe that Dirichlet energy measurements are consistent with the action of this implicit Geometric Occam's Razor.

\end{abstract}

\section{Introduction}\label{section:intro}\label{section:motivation}

Naively, we might expect that over-parameterized models will overfit the training data and that under-parameterized models will be better since they have fewer degrees of freedom. However, it turns out that over-parameterized models can find better solutions than the under-parameterized models - a  paradoxical phenomenon known as the double-descent curve \cite{belkin2021fear,Nakkiran2020Deep}. One possible explanation for this behaviour is that over-parameterized models are subject to an Occam's razor that filters out unnecessarily complex solutions in favour of simpler solutions. 

Typically we might expect an Occam's razor to take the form of a complexity measure on the number of model parameters or the size of the hypothesis space, for instance \cite{generalization}. However, for neural networks, the precise form of this hypothesized Occam's razor is not known, since it is not explicitly enforced during training. There has been some progress recently to identify sources of implicit regularization that may play a role here \cite{barrett2021implicit, smith2021on,chao2021sobolev,blanc2020}. For instance, recent work has exposed a hidden from of regularization in Stochastic Gradient Descent (SGD) called Implicit Gradient Regularization (IGR) \cite{barrett2021implicit, smith2021on} which penalizes learning trajectories that have large loss gradients. 

For over-parameterized neural networks trained with SGD, we hypothesize that the hidden Occam's razor takes the form of a geometric complexity measure. Our key contributions are as follows: (1) define this notion of geometric complexity; (2) show that the Dirichlet energy can be used as a proxy for geometric complexity; (3) show that the IGR mechanism from SGD puts a regularization pressure on the geometric complexity; and (4) show that the strength of this pressure increases with the size of the learning rate, which we verify with numerical experiments. 

\section{The Geometric Occam's razor in 1-dimensional regression}

 To build intuition, we begin with a simple 1-dimensional example. Consider a ReLU neural network consisting of 3 layers with 300 units per layer, trained using SGD, without any form of explicit regularization to perform 1-dimensional regression using only 10 data points. In this extreme setting, we should expect the network to overfit the dataset, since the function space described by that neural network is extremely large - consisting of piecewise linear functions with thousands of linear pieces \citep{arora2018understanding}.  Yet, if we plot the learned function during training from the first step all the way up to interpolation, as in Figure \ref{fig:training movie}, we observe that the learned function is the `simplest' possible function, in some sense, among all functions with the same training error. 

But what do we mean by `simple'? Our key intuition in this example is that the arc length of the learned function over the smallest interval containing the data points provides our measure of model complexity. At the end of training, we see that the arc length of the learned function is close to that of the shortest possible path interpolating between the data points, suggesting that this measure of geometric complexity is somehow optimised during training.

\begin{figure}[t]
  \centering
  \includegraphics[width=1\linewidth]{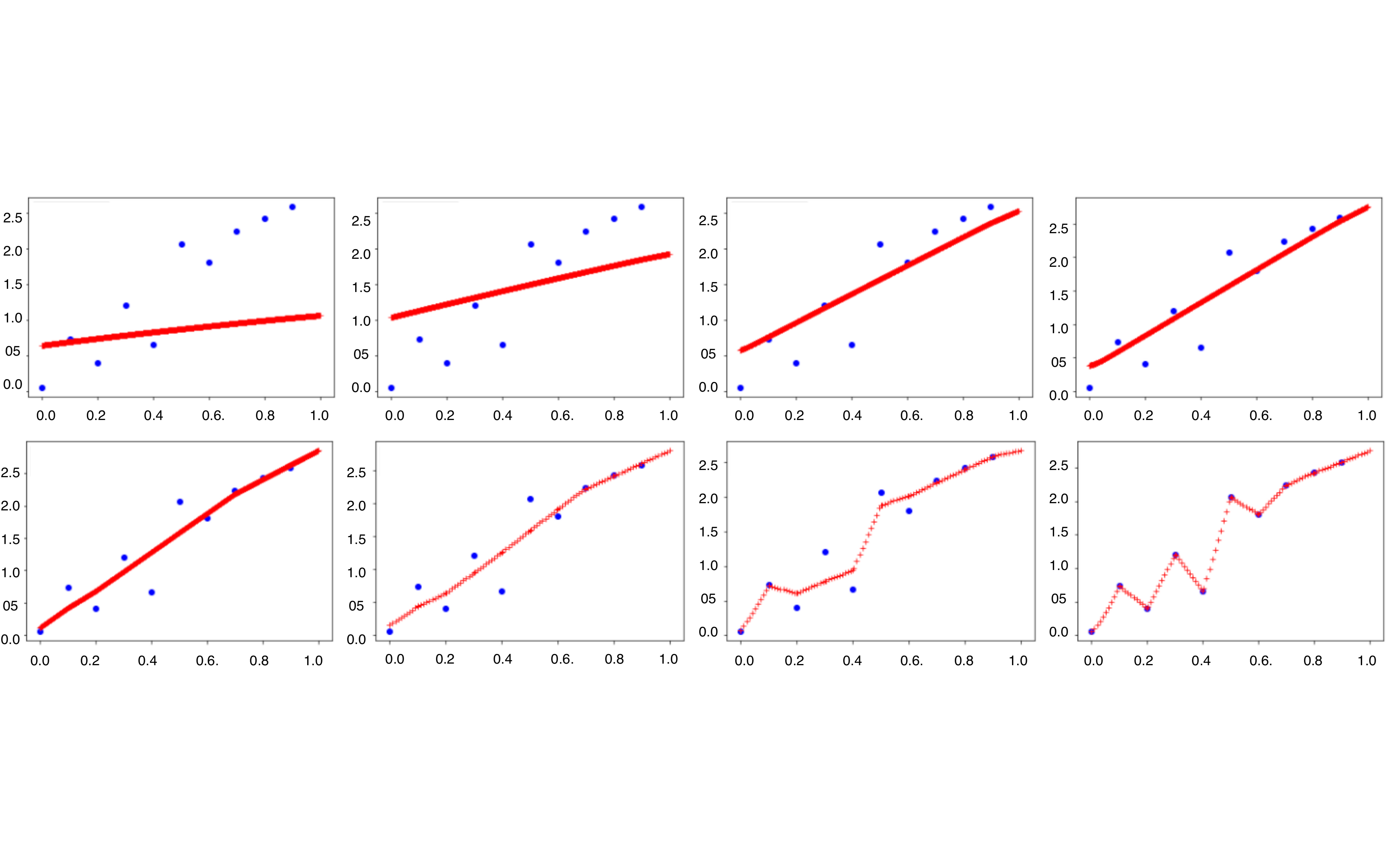}
  \vspace{-2.25cm}
  \caption{Training sequence for an over-parameterised neural network trained using 10 data points plotted at iteration step 1, 20, 40, 50, 5000, 10000, 20000, and 30000 (from top left to bottom right, respectively). The blue circles represent the data points while the red crosses represent the points predicted by the neural network.}
  \label{fig:training movie}
\end{figure}

\section{Dirichlet energy as a measure of function complexity}\label{section:dirichlet_energy}

In the previous section, we used the arc length of the learnt 1-dimensional function as a measure of its geometric complexity. What is the corresponding notion for a function in a high-dimensional feature space $f:\R^d \rightarrow \R$? In this case, we can define the {\it geometric complexity} of a function $f$ as the volume $ \Omega_D(f)$ of its graph
\begin{equation}\label{equation:graph}
\textrm{gr}(f_{X_D}) := \{(x, f(x)):\: x\in X_D\} \subset \R^d \times \R
\end{equation}
restricted to the feature polytope $X_D$; that is, the polytope with smallest volume containing all the feature points $x_i$ of the dataset $D = \{(x_i, y_i): i=1,\dots, n\}$.
From differential geometry \citep{differential1976carmo}, for a smooth function $f:\R^d \rightarrow \R$, the graph $\textrm{gr}(f_{X_D})$ is an $n$-dimensional smooth submanifold of $\R^{n+1}$. Using the Riemannian metric on $\textrm{gr}(f_{X_D})$ induced from the Euclidean metric on $\R^{n+1}$ and its corresponding Riemannian volume form, the  volume of the graph of $f$ can be expressed as
\begin{equation}\label{equation:volume}
\Omega_{D}(f) = \int_{X_D} \sqrt{1 + \| \nabla_x f \|^2} ~dx.
\end{equation}
This can in turn be approximated using a first-order Taylor series expansion $\sqrt{1 + z} \simeq 1 + \frac 12 z$ so that
\begin{equation}\label{equation:approx}
\Omega_{D}(f) \approx \int_{X_D} 1 + \frac 12 \| \nabla_x f \|^2 ~dx = 
\textrm{Vol}(X_D) + \textrm{DE}(f),
\end{equation}
where 
\begin{equation}\label{equation:DE}
\textrm{DE}(f) = \frac{1}{2} \int_{X_D} \|\nabla_x f\|^2 dx.
\end{equation}
is the {\it Dirichlet energy } of the function $f$ over $X_D$.
The computation above suggests that both the function volume or its Dirichlet energy can be used as a measure of a function's geometric complexity. 

One way to compute the Dirichlet energy numerically is to use a quadrature formula summing up $\|\nabla_x f\|^2$ over a number of points in $X_D$ and multiplying the summands by the volume element of the point. So, if we use the data points themselves for evaluation and $1/|D|$ as a proxy for the volume element, we obtain a discrete version of the Dirichlet energy which we call the {\it discrete Dirichlet energy}, denoted by $\widehat{\textrm{DE}}(f)$. This provides an easily computable measure of a function's geometric complexity:
\begin{equation}\label{equation:de}
\widehat{\textrm{DE}}(f) = \frac 1{2|D|} \sum_{x\in D}\|\nabla_x f(x)\|^2.
\end{equation}

\section{How neural networks tame model complexity}\label{section:implicit_regularization}

We now argue that the geometric complexity, as measured by the discrete Dirichlet energy, is implicitly regularized during the training of neural nets with vanilla SGD. 

In recent work \cite{smith2021on}, it was shown that the discrete steps of SGD from epoch to epoch closely follow, on average, the gradient flow of a modified loss of the form:
$$
\tilde L = L + 
\frac h{4|D|}\sum_{(x,y)\in D} \|\nabla_\theta L(x, \theta)\|^2,
$$
where $L$ is the original loss and $L(x, \theta) := E(f_\theta(x), y)$ is the error $E$ between the prediction $f_\theta(x)$ and the true label $y$. This means that during SGD the quantities $\|\nabla_\theta L(x, \theta)\|^2$  at each data point $(x, y)\in D$, are implicitly regularized, with the learning rate $h$ acting as an implicit regularization rate. 

Now, for models whose losses come from the application of a maximum likelihood estimation on a conditional probability distribution in the exponential family such as the least-square loss or the cross-entropy loss, we obtain loss gradients that have the following form:
$$
\nabla_\theta L(x, \theta) = \epsilon_x(\theta) \nabla_\theta f_\theta(x),
$$
where $\epsilon_x(\theta) = (f_\theta(x) - y)$ is the signed residual, yielding
\begin{equation}\label{equation:resdidual_loss}
\tilde L = L + 
\frac h{4|D|}\sum_{(x,y)\in D}\epsilon_x(\theta)^2\|\nabla_\theta f_\theta(x)\|^2.
\end{equation}
From that last expression for $\tilde L$, we see that the terms $\|\nabla_\theta f_\theta(x)\|^2$ are implicitly regularized at each data point $(x,y)$ and even more so in the region where the residual errors are large, such as the beginning of training.

We now argue that for {\it neural networks}, in particular, the regularization pressure on the gradient of the network with respect to the parameters $\|\nabla_\theta f_\theta(x)\|^2$ acts as a regularization pressure on the gradient of the network with respect to the input $\|\nabla_x f_\theta(x)\|^2$. Hence, this creates a pressure for the Dirichlet energy to be implicitly regularized during training. In fact, this follows from the fact that for neural networks their derivatives with respect to the inputs and the parameters can be related as follows (proof in Appendix \ref{appendix:proof}):

\begin{theorem}
Consider a neural network with $l$ layers
$
f_\theta(x) = f_1\circ \dots \circ f_l(x),
$
where $f_i(z) = a_i(w_i z + b_i)$ with $\theta = (w_1, b_1, \dots, w_l, b_l)$ being the vector of layer weight matrices $w_i$ and biases $b_i$ and the $a_i$'s are the layer activation functions. Then we have that
\begin{equation}\label{equation:gradient_relation}
    \|\nabla_x f_\theta(x)\|^2 \left(
        \frac{1 + \|h_1(x)\|^2}{\|w_1\|^2\|h_1'(x)\|^2}
            + \cdots + 
        \frac{1 + \|h_l(x)\|^2}{\|w_l\|^2\|h_l'(x)\|^2}
    \right) \leq \|\nabla_{\theta}f(x)\|^2,
\end{equation}
where $h_i(x)$ is the sub-network from input $x$ to layer $i$ and $\|w_i\|$ is the spectral norm of the weight matrix $w_i$.
\end{theorem}

From \eqref{equation:gradient_relation}, we see that the regularization pressure from IGR translates into a regularization pressure on the discrete Dirichlet energy when the positive quantities 
\begin{equation}\label{equation:Ai}
A_i(\theta, x) := \|w_i\|\|h_i'(x)\|, \quad i=1,\dots, l
\end{equation}
remain small. Note that this is expected to happen at the beginning of training when the spectral norms of the layers are close to zero, while they tend to grow as the training progresses if no spectral regularization \cite{miyato2018spectral} is applied. Furthermore, note here that preventing the $A_i$'s from becoming too large during training may be an important consideration which informs the choice of model architecture and layer regularization.

{\bf Experimental evidence:}  From Equation \eqref{equation:resdidual_loss}, since the strength of IGR is a function of the learning rate, we should expect an increased pressure on the Dirichlet energy as a result of Equation \eqref{equation:gradient_relation} when training with higher learning rates. We verify this prediction for a ResNet-18 trained to classify CIFAR-10 images. Measuring the discrete Dirichlet energy at the time of maximal test accuracy for a range of learning rates,
we observe this predicted behaviour, consistent with our theory; see Figure \ref{figure:de_vs_lr}.

\begin{figure}[t]
  \centering
  \includegraphics[width=0.3\linewidth]{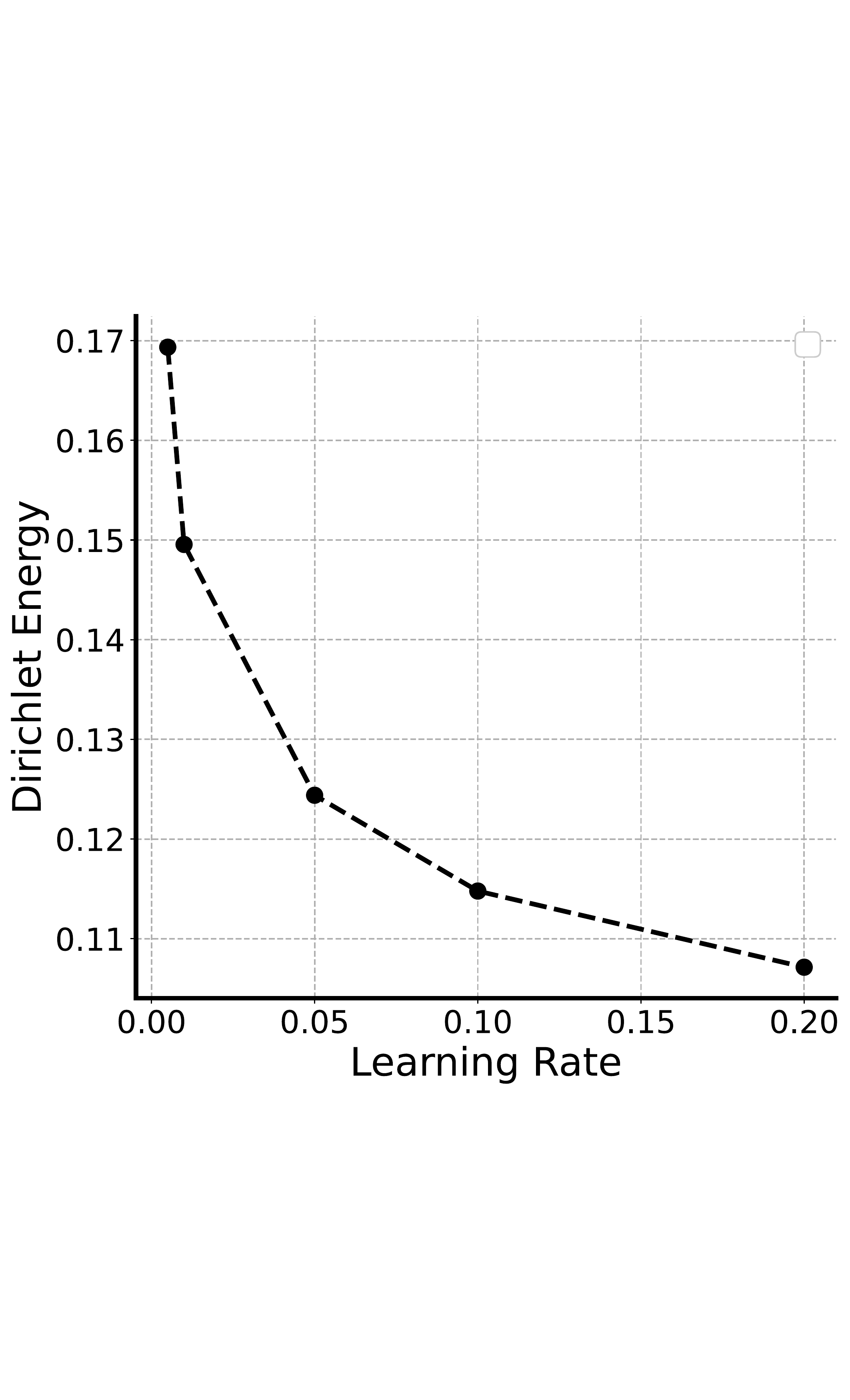}
  \includegraphics[width=0.3\linewidth]{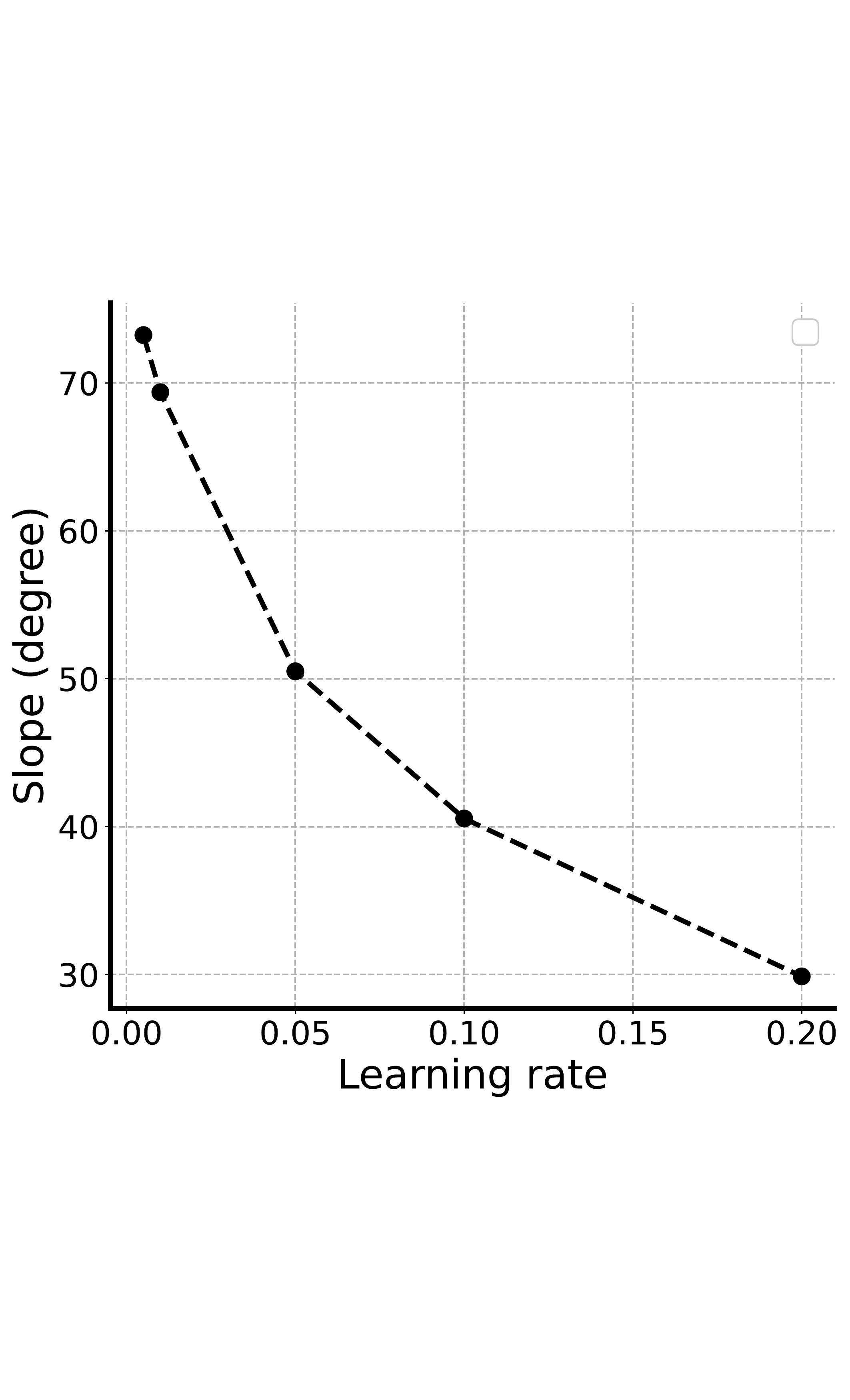}
  \includegraphics[width=0.3\linewidth]{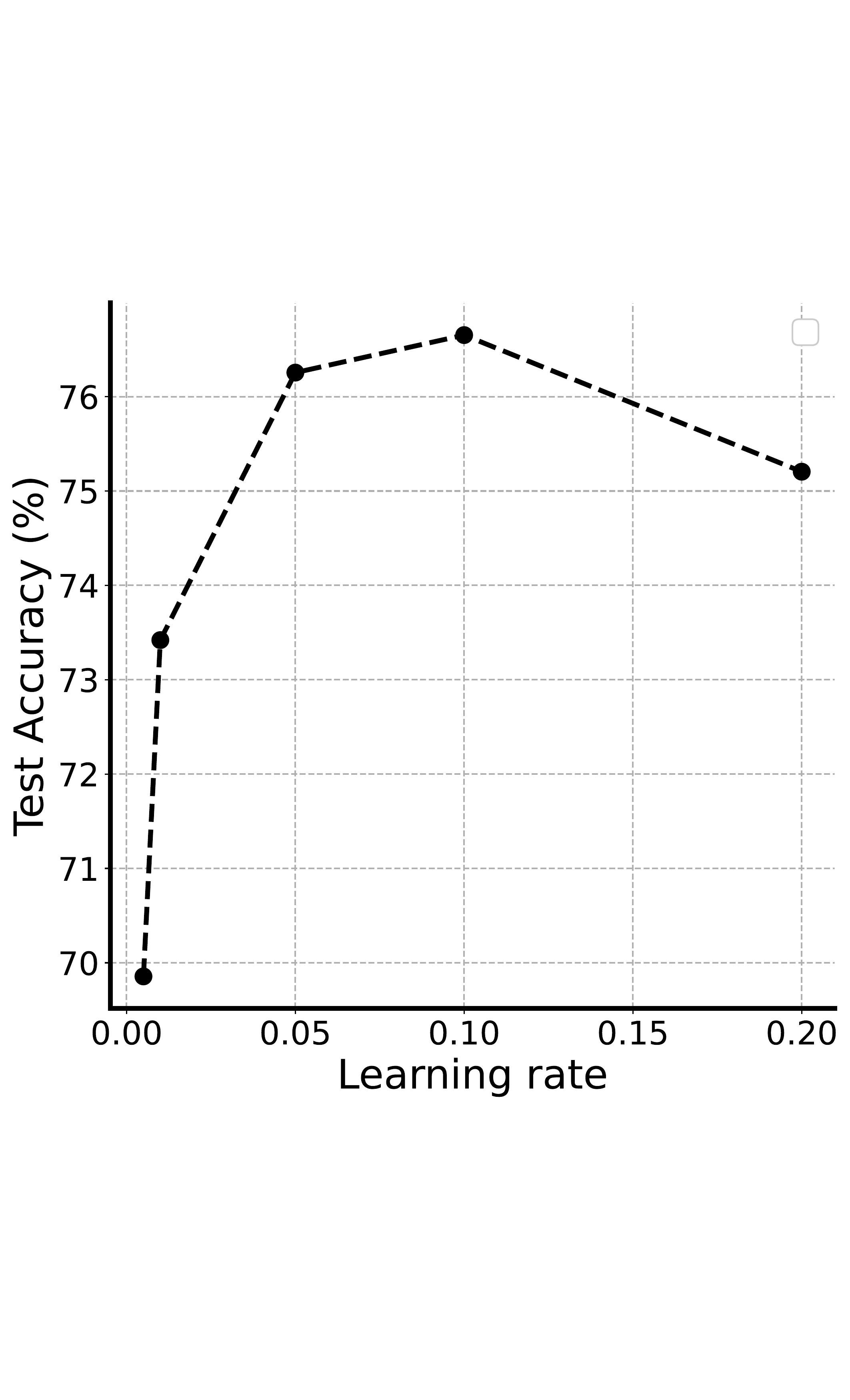}
  \vspace{-1.25cm}
  \caption{For a ResNet-18 trained on CIFAR-10, we observe that the Dirichlet energy and the loss-surface slope (essentially the IGR term $\|\nabla_\theta L\|^2$) decrease as the learning rate increases. Each dot corresponds to a fully trained model at time of maximal test accuracy.}
  \label{figure:de_vs_lr}
\end{figure}

\begin{remark}
Note also that for linear models (i.e., neural networks with a single linear layer), the Dirichlet energy of the network coincides in this case with the L2-norm of the parameters. Therefore, this results recovers the already known fact that linear models trained with SGD have an inductive bias towards low L2-norm solutions (see \cite{generalization}). This also points toward the fact that the Dirichlet energy may be the right generalization of the L2-norm for a general network. 
\end{remark}

\section{Related work, Future directions, and Discussion}
{\bf Splines and connections to harmonic function theory:} The Dirichlet energy \eqref{equation:DE} is well-known in harmonic function theory \citep{axler2013harmonic} where it can be shown using calculus of variations that harmonic functions subject to a boundary condition minimize the Dirichlet energy over the space of differentiable functions. This is known as Dirichlet's principle. The minimization of the Dirichlet energy itself is also related to the theory of splines \cite{jin2021implicit}. Our work seems to indicate that neural networks are biased towards (a notion of) harmonic functions with the dataset acting as the boundary condition. {\bf Complexity theory:} The notion of geometric complexity introduced has similarities to the Kolgomorov complexity \cite{Schmidhuber1997DiscoveringNN} as well as the minimum description length given in \cite{NIPS1993_9e3cfc48}. {\bf Smoothness regularization:} The notion of geometric complexity considered here is related to the notion of smoothness with respect to the input as discussed in \cite{pmlr-v137-rosca20a} as well as to the Sobolev regularization effect of SGD discussed in \cite{chao2021sobolev}, where inequalities similar to \eqref{equation:gradient_relation} but involving only the first layer are considered.
In particular, various forms of gradient penalties, reminiscent of the Dirichlet energy, have been devised to achieve Lipschitz smoothness \cite{NEURIPS2018_42998cf3, gulrajani2017improved, fedus2017many, NEURIPS2018_07f75d91, kodali2018on}. It has been shown that the discrete Dirichlet energy (evaluated at the data points) is a powerful regularizer \cite{hoffman2020robust, varga2018gradient} and in \cite{pmlr-v137-rosca20a} that it has advantages over other form of smoothness regularization (such as spectral norm regularization \cite{miyato2018spectral,zinan2020why}). Our analysis shows that we can control this form of regularization cheaply through the learning rate. In image processing, the Dirichlet energy is also called the Rudin–Osher–Fatemi total variation and it as been introduced as a powerful explicit regularizer for image denoising; see \cite{RUDIN1992259} and \cite{getreuer2012}. It may be possible that these various forms of smoothness regularization are useful because they provide implicit control over the model geometric complexity. {\bf Regularization through noise:}
The discrete Dirichlet energy is reminiscent of the Tikhonov regularizer which is implicitly regularized with added input noise \cite{bishop1995}. The modified loss in \eqref{equation:resdidual_loss} is also very reminiscent of the modified loss in \cite{blanc2020}, which is argued to be implicitly minimized by SGD when a random white noise is added to the labels. In Section 3 of \cite{pmlr-v80-mescheder18a}, it is argued that explicit gradient regularization with respect to input and noise instance produce similar types of regularization. 
Altogether, this suggests that feature noise, label noise, and the optimization scheme all conspire to implicitly tame the geometric complexity in the case of neural networks trained with gradient-based optimization schemes. {\bf Regularization through the number of layers:} In Equation \eqref{equation:gradient_relation}, one sees that each layer contributes an additional positive term, increasing the pressure on the Dirichlet energy. This suggests that the pressure on the model geometric complexity may increase with the neural network depth in a similar spirit as \cite{gao2016degree}. {\bf Training of GANs:} For GANs, explicit gradient regularization both with respect to the input \citep{gulrajani2017improved,NEURIPS2018_07f75d91,fedus2017many,kodali2018on,miyato2018spectral} and the parameters \cite{rosca2021discretization,odegan,balduzzi2018mechanics,mescheder2017numerics,nagarajan2017gradient} has been proven to be beneficial and related to smoothness. 
Our main theorem provides a way to relate gradient penalties with respect to the input and with respect to the parameters for neural networks, in a way where the spectral norm of the weight matrices plays a key role. This points toward geometric complexity being a useful notion to relate and understand these different forms of regularization (including spectral normalization as in \cite{miyato2018spectral} and \cite{zinan2020why}).

\section{Conclusion}
In conclusion, we have found that neural networks trained with SGD are subject to an implicit Geometric Occam's razor, which selects parameter configurations that have low geometric complexity ahead of configurations with high geometric complexity. This geometric complexity is given by the arc length in 1-dimensinal regression; is linearly related to the Dirichlet energy in higher-dimensional settings; and has many intriguing similarities to other known quantities, including various forms of implicit and explicit regularisation and model complexity.  More generally, our work develops promising new theoretical connections between optimization and the geometry of over-parameterised  neural networks.

\subsubsection*{Acknowledgments}

We would like to thank Mihaela Rosca, Maxim Neumann, Yan Wu, Samuel Smith, and Soham De for helpful discussion and feedback. We would also like to thank Patrick Cole and Shakir Mohamed for their support.

\newpage
\appendix

\section{Proof of Theorem 1} \label{appendix:proof}
Consider a neural network with $l$ layers
$
f_\theta(x) = f_1\circ \dots \circ f_l(x),
$
where $f_i(z) = a_i(w_i z + b_i)$ with $\theta = (w_1, b_1, \dots, w_l, b_l)$ being the vector of layer weight matrices $w_i$ and biases $b_i$ and the $a_i$'s are the layer activation functions. We will use the notation $f_{w_i}$ or $f_{b_i}$ instead of $f_\theta$ when we want to consider $f_\theta$ as dependent on the $w_i$'s or $b_i$'s only.

For this model structure and following Pythagoras, we have:
\begin{equation}\label{equation:decomposition}
\|\nabla_\theta f_\theta(x)\|^2 = 
\|\nabla_{w_1} f_{\theta}(x)\|^2 + \|\nabla_{b_1} f_{\theta}(x)\|^2 
+ \cdots + 
\|\nabla_{w_l} f_{\theta}(x)\|^2 + \|\nabla_{b_l} f_{\theta}(x)\|^2 
\end{equation}
For each layer $i$, we can rewrite the network function as $$f_{w_i}(x) = g_i(w_i h_i(x) + b_i),$$ where $g_i(z)$ consists of the deeper layers above $i$ and $h_i(x)$ consists of the shallower layers below $i$. The idea now, inspired from \cite{SeongLKHK18}, is to show that a small perturbation $\delta x$ of the input is equivalent to a small perturbation $u(\delta x)$ of the weights of layer $i$. We will use this idea to prove the following two lemmas.

\begin{lemma} In the notation above, for each layer $i$, we have:
\begin{equation}\label{equation:weight_derivative}
    \|\nabla_x f_\theta(x)\|^2 \left(\frac{\|h_i(x)\|}{\|w_i\|\|h_i'(x)\|}\right)^2 \leq \|\nabla_{w_i}f(x)\|^2.
\end{equation}
\end{lemma}

\begin{proof}
Consider a small perturbation $x + \delta x$ of the input $x$. We start by showing that we can always find a corresponding perturbation $w_i + u(\delta x)$ of the weight matrix in layer $i$ such that
\begin{equation}\label{equation:equivalence}
f_{w_i}(x + \delta x) = f_{w_i + u(\delta x)}(x).
\end{equation}
Namely, because $f_{w_i}(x) = g_i(w_i h_i(x) + b_i)$, to show this, it is enough to find $u(\delta x)$ such that
\begin{equation}\label{equation:condition}
w_i(h_i(x) + h_i'(x)\delta x ) + b_i = (w_i + u(\delta x)) h_i(x) + b_i,
\end{equation}
where we identify  $h_i(x + \delta x)$ with its linear approximation around $x$ for small $\delta x$. Then \eqref{equation:condition} is always satisfied if we set
\begin{equation}\label{equation:perturbation}
    u(\delta x) := \frac{(w_i h_i'(x) \delta x)h_i^T(x)}{\|h_i(x)\|^2},
\end{equation}
since $h_i^T(x)h_i(x) = \|h_i(x)\|^2$. Now taking the derivative with respect to $\delta x$ at $\delta x = 0$ on both sides of Equation \eqref{equation:equivalence}, and using the chain rule and that $u(\delta x)$ is linear in $\delta x$, we obtain a relation between the network derivative with respect to the weight matrices and w.r.t the input:
\begin{equation}
    \nabla_x f_\theta(x)  = \nabla_{w_i} f(x) \frac{(w_i h_i'(x))h^T(x)}{\|h_i(x)\|^2}.
\end{equation}
Taking the norm on both sides, squaring, and rearranging the terms yields \eqref{equation:weight_derivative}.
\end{proof}

Following the same strategy, we now prove a corresponding lemma for the biases at each layer:

\begin{lemma} In the notation above, for each layer $i$, we have:
\begin{equation}\label{equation:bias_derivative}
    \|\nabla_x f_\theta(x)\|^2 \left(\frac{1}{\|w_i\|\|h_i'(x)\|}\right)^2 \leq \|\nabla_{b_i}f(x)\|^2.
\end{equation}
\end{lemma}

\begin{proof}
Consider a small perturbation $x + \delta x$ of the input $x$. We start again by showing that we can always find a corresponding perturbation $b_i + u(\delta x)$ of the biases in layer $i$ such that
\begin{equation}\label{equation:equivalence}
f_{b_i}(x + \delta x) = f_{b_i + u(\delta x)}(x).
\end{equation}
Namely, because $f_{b_i}(x) = g_i(w_i h_i(x) + b_i)$, to show this, it is enough to find $u(\delta x)$ such that
\begin{equation}\label{equation:condition2}
w_i(h_i(x) + h_i'(x)\delta x ) + b_i = w_i h_i(x) + b_i + u(\delta x),
\end{equation}
where we again identify $h_i(x + \delta x)$ with its linear approximation around $x$ for small $\delta x$. Then \eqref{equation:condition2} is always satisfied if we set this time
\begin{equation}\label{equation:perturbation}
    u(\delta x) := w_i h_i'(x) \delta x.
\end{equation}
Now taking the derivative with respect to $\delta x$ at $\delta x = 0$ on both sides of Equation \eqref{equation:equivalence}, and using the chain rule and that $u(\delta x)$ is linear in $\delta x$, we obtain a relation between the network derivative w.r.t. the biases and w.r.t the input:
\begin{equation}
    \nabla_x f_\theta(x)  = \nabla_{w_i} f(x) w_i h_i'(x).
\end{equation}
Taking the norm on both sides, squaring, and rearranging the terms yields \eqref{equation:bias_derivative}.
\end{proof}

Using this in quality for each layer in the decomposition given by Equation \eqref{equation:decomposition}, we obtain that
\begin{equation}
    \|\nabla_x f_\theta(x)\|^2 \left(\left(\frac{k_1(x)}{\|w_1\|}\right)^2 
       + \cdots + 
    \left(\frac{k_l(x)}{\|w_l\|}\right)^2\right) \leq \|\nabla_{\theta}f(x)\|^2
\end{equation}

Now if we use \eqref{equation:weight_derivative} and \eqref{equation:bias_derivative} in \eqref{equation:decomposition}, we finally obtain that
\begin{equation}
    \|\nabla_x f_\theta(x)\|^2 \left(
        \frac{1 + \|h_1(x)\|^2}{\|w_1\|^2\|h_1'(x)\|^2}
            + \cdots + 
        \frac{1 + \|h_l(x)\|^2}{\|w_l\|^2\|h_l'(x)\|^2}
    \right) \leq \|\nabla_{\theta}f(x)\|^2.
\end{equation}
\end{document}